\documentclass[12pt,a4paper,sn-mathphys-num]{sn-jnl}
\usepackage{graphicx}%
\usepackage{multirow}%
\usepackage{amsmath,amssymb,amsfonts,amsthm}%
\usepackage{mathrsfs}%
\usepackage{newtxmath,newtxtext}
\usepackage[title]{appendix}%
\usepackage{xcolor}%
\usepackage{textcomp}%
\usepackage{booktabs}%
\usepackage{algorithm}%
\usepackage{algorithmicx}%
\usepackage{algpseudocode}%
\usepackage{listings,siunitx}%
\usepackage[numbers,sort&compress]{natbib}
\theoremstyle{thmstyleone}%
\newtheorem{theorem}{Theorem}

\theoremstyle{thmstyletwo}%

\theoremstyle{thmstylethree}%
\newtheorem{definition}{Definition}%

\begin{document}
	\title{Dynamic multi-objective lion swarm optimization with multi-strategy fusion: An application in 6R robot trajectory planning}
\author[1]{\fnm{Bao} \sur{Liu}}\email{liubao@upc.edu.cn}

\author[1]{\fnm{Tianbao} \sur{Liu}}\email{cnliutb@163.com}

\author[1]{\fnm{Zhongshuo} \sur{Hu}}\email{whlov112@163.com}

\author[1]{\fnm{Fei} \sur{Ye}}\email{yefei17860707230@163.com}

\author*[2,]{\fnm{Lei} \sur{Gao}}\email{lei.gao@csiro.au}

\affil[1]{\orgdiv{College of control science and engineering}, \orgname{China University of Petroleum (East China)}, \orgaddress{\city{Qingdao}, \postcode{266580}, \country{China}}}

\affil*[2]{\orgdiv{Commonwealth Scientific and Industrial Research Organisation (CSIRO)}, \orgname{Waite Campus}, \orgaddress{\city{Urrbrae}, \postcode{SA 5064}, \country{Australia}}}

\abstract{The advancement of industrialization has spurred the development of innovative swarm intelligence algorithms, with Lion Swarm Optimization (LSO) notable for its robustness, parallelism, simplicity, and efficiency. While LSO excels in single-objective optimization, its multi-objective variants face challenges such as poor initialization, local optima entrapment, and slow adaptation to dynamic environments. This study proposes Dynamic Multi-Objective Lion Swarm Optimization with Multi-strategy Fusion (MF-DMOLSO) to address these limitations. MF-DMOLSO comprises three key components: initialization, swarm position update, and external archive update. The initialization unit employs chaotic mapping for uniform population distribution. The position update unit enhances behavior patterns and step size formulas for cub lions, incorporating crowding degree sorting, Pareto non-dominated sorting, and Levy flight mutation to improve convergence speed and global search capabilities. Reference points guide convergence in higher-dimensional spaces, maintaining population diversity. An adaptive cold-hot start strategy generates an initial population responsive to environmental changes. The external archive update unit re-evaluates solutions based on non-domination and diversity to form the new population. Evaluations on benchmark functions showed MF-DMOLSO outperformed existing algorithms, exceeding 90\% accuracy for two-objective and 97\% for three-objective problems. MF-DMOLSO surpassed multi-objective particle swarm optimization, non-dominated sorting genetic algorithm II, and multi-objective lion swarm optimization. Compared to non-dominated sorting genetic algorithm III, MF-DMOLSO showed a 60\% improvement. Applied to 6R robot trajectory planning, MF-DMOLSO optimized running time and maximum acceleration to 8.3s and 54°/s², achieving a set coverage rate of 70.97\% compared to 2\% by multi-objective particle swarm optimization, thus improving efficiency and reducing mechanical dither.}
\keywords{Lion swarm optimization; Crowding sort; Levy mutation; Reference point mechanism; Adaptive cold and hot start strategy; Dynamic multi-objective optimization}

\maketitle
	\section{Introduction}
The increasing demand for multi-objective optimization algorithms that can deliver optimal solutions with superior convergence rates and high success rates across various domains, including multi-objective decision-making, robot trajectory planning, and neural network parameter training, has led to a surge of interest \cite{yue2021differential}. Algorithms such as multi-objective particle swarm optimization (MOPSO) or non-dominated sorting genetic algorithm II (NSGA-II) have been successfully applied in addressing problems in diverse domains, including robot path planning \cite{mac2017hierarchical}, energy storage \cite{geng2019energy}, mine water management \cite{gao2014systems}, and even military applications \cite{bukar2019optimal}. 

Compared to single-objective optimization methods, multi-objective evolutionary algorithms (MOEAs) typically exhibit higher computational complexities and more stringent convergence criteria. Additionally, many optimization environments change dynamically, further complicating the problem. Aboud et al. \cite{aboud2022dpb} proposed a Dynamic Pareto bi-level MOPSO (DPb-MOPSO) algorithm for dynamic multi-objective problems using distributed architecture-based particle swarm optimization. NSGA-II can optimize the problem with two objectives, but struggles with high-dimensional objective space. Some scholars combined algorithms such as NSGA-II and MOGWO to address practical challenges \cite{rostamzadeh2024multi,cheraghi2023multi}, leveraging the advantages of multiple algorithms to broaden the solution space. However, these improvements remain insufficient in multiple objective spaces. Deb et al. \cite{deb2013evolutionary} proposed a non-dominated sorting genetic algorithm III (NSGA-III), which improved the convergence and distribution of populations by introducing the expression of population diversity based on reference points. Although NSGA-III is the most classical algorithm for high-dimensional target optimization at present, it still faces difficulty in approaching global optimization and loss of evolutionary diversity when dealing with constraint problems and high-dimensional problems \cite{gu2022improved}. Meanwhile, NSGA-III requires the establishment of reference points even for two-dimensional problems, adding complexity and redundancy to the optimization process. Purshouse and Fleming \cite{purshouse2007evolutionary} underscored the importance of activating diversity mechanisms. They demonstrated that the mere existence of diversity mechanisms is insufficient and that the active participation of these mechanisms matters. Maintaining diversity near the Pareto Frontiers (PFs) and achieving convergence are two conflicting goals, while a common genetic operator is insufficient to achieve both simultaneously. Therefore, enhancing the approximation to the true PFs and a well-distributed population remain the common goals for existing MOEAs.

The above issues have hindered the further improvement of MOEAs in dealing with multi-local optimal functions. However, the Lion Swarm Optimization (LSO) algorithm, characterized by evolutionary diversification and a balance between early-stage exploration and later-stage exploitation, may provide a novel approach to further improve MOEAs. LSO proposed by Liu et al. \cite{liu2018swarm} in 2018, is a swarm intelligence algorithm inspired by the foraging behavior of lions in nature. In contrast to the standard MOPSO, which relies on a single individual updating mode and often exhibits poor global searching abilities, LSO adopts a cooperative foraging method observed in lion groups. This approach diversifies the updating of individual positions, enabling rapid algorithm convergence and reducing the likelihood of getting trapped in local optima. 

Additionally, LSO demonstrates reduced reliance on experience-based parameter settings, enhancing its universality. Due to its advantages in terms of fast convergence and well-defined structure, LSO has garnered increasing attention \cite{liu2020lion,huang2024multi}. Although LSO has demonstrated remarkable performance in various fields \cite{qiao2020hybrid}, further efforts are required to explore its potential in multi-objective optimization.

Despite the integration of reverse elite learning into LSO, the majority of individuals still adhere to the principle of learning from the elite. In situations where no individual manages to reach the forefront of actual optimal solutions or if the optimal solution is not part of the reverse elite learning path, there remains a risk of the algorithm converging towards a local optimum. Notably, when dealing with high parameter dimensions and multiple objectives, it becomes challenging to avoid converging towards pseudo-PFs or part of the real PFs. Dong et al. \cite{dong2020joint} introduced a resampling mechanism in the later stages of multi-objective lion swarm optimization to increase the number of individuals in sparse areas, thereby improving the spatial distribution of the population and significantly enhancing global search capabilities. However, during the initial optimization period, they lacked effective guidance for population convergence, leading to a slightly slower convergence speed. Additionally, there is some randomness in updating individual and global optima, and the guiding influence of elites is not as significant as it could be. Ji et al. \cite{ji2021lion} proposed a lion swarm optimization by reinforcement pattern search, which effectively addressed issues such as lions easily going out of bounds when operating in a large range and the non-universality of position update formulas. This approach simultaneously enhanced the algorithm's local search capability. Nevertheless, the algorithm still tended to get trapped in local optimal solutions, indicating that its global search capability remains to be improved. Liu et al. \cite{liu2023adaptive} proposed an adaptive lion swarm optimization algorithm incorporating chaotic search and information entropy to solve the problem that LSO algorithm has a slow convergence speed and tends to fall into the local optimal in subsequent iterations. The accuracy and stability of the algorithm were proved to be excellent, but its advantages were not extended to multi-objective problems. 

To solve the problems mentioned above, this paper proposes the Dynamic Multi-Objective Lion Swarm Optimization with Multi-strategy Fusion (MF-DMOLSO). MF-DMOLSO consists of three main components: initialization unit, swarm position update unit, and external archive update unit. The initialization unit employs chaotic mapping, which exhibits ergodicity and long-term unpredictability, ensuring a more uniform distribution of individuals within the search space and accelerating the discovery of PFs. In the position update unit, the selection strategy combines Pareto dominance relationships with population diversity metrics. Diversity assessment methods are introduced respectively for different dimensions: plane crowding distance for two-dimensional objective and the spatial crowding degree based on reference points for multi-objective. Furthermore, this paper introduces Levy flight mutation, enhancing the population's ability to escape local optima and improving convergence speed. To cope with dynamic targets, adaptive cold and hot start strategies are also introduced to adapt to environmental changes and generate new populations based on these changes and historical information. Additionally, this paper modifies the evolutionary approach for cub lions. The external archive update unit selects the next generation of individuals and archive solutions based on population dominance relationships and diversity, while also identifying the global optimal solution.

The subsequent sections of this paper are organized as follows. Section 2 presents a fundamental description of LSO. Section 3 introduces in detail the enhanced strategies integrated into MF-DMOLSO. Section 4 outlines the experimental methodology and parameter settings. Section 5 presents the performance of MF-DMOLSO on benchmark functions. Section 6 focuses on the piecewise polynomial trajectory planning of a 6R robot and the optimization of trajectory parameters using MF-DMOLSO. Section 7 discusses the advantages and limitations of the proposed algorithm. Finally, Section 8 concludes the research efforts of the entire paper.

\section{Lion swarm optimization}
In recent years, LSO has garnered increasing attention in the field of intelligent optimization algorithms due to its diverse evolutionary mechanism and excellent optimization performance. This algorithm introduces the concepts of gender and age to the population, emulating the internal structure of a lion swarm. This results in distinct evolutionary paths for different individuals, fostering a highly collaborative environment.

Within the pride's territory, the lion king assumes responsibility for protecting the territory, feeding the young lions, and fending off intruding lions. Typically, the largest prey within the territory is claimed by the king, who serves as a proxy for the global historical optimum in the algorithm. Lionesses collaborate to hunt prey. If they encounter a more lucrative target than the current territorial king, they will usurp it. Young lions, or follower lions, primarily interact with adult lions in three ways: following the king to feed, shadowing lionesses to learn hunting techniques, or being expelled from the territory to fend for themselves and eventually challenge the king's status upon returning \cite{liu2018swarm}.

Let the target $f_i(x),\ i=1,2,3\ldots$. The number of lions is N, the maximum number of iterations of the algorithm is T, and the solution of the optimization goal is in a D-dimensional space, namely $x_j=(x_{j1},x_{j2},x_{j3},\ldots x_{jD}), 1\le j\le N. \beta is the adult lion ratio, \beta\in(0,1)$. Since adult lions pay more attention to local optimization, the adult lion ratio is usually set below 0.5, so that there is a larger number of young lions to maintain better population diversity and more diversified evolution ways and improve the global detection ability. The following paper will introduce the evolutionary mechanism of different lion species in detail.

(1) The evolutionary mechanism of the lion king

The lion king represents the current historical optimal value, so the lion king chooses to move in a small range near the current optimal position to explore whether there is a better solution near the current optimal value. Its updating mode is as follows:
\begin{equation}
	X_i^{k+1}=g^k(1+\gamma\|p_i^k-g^k\|)
\end{equation}

(2) Lioness evolution mechanism

Lionesses need to cooperate with another lioness during hunting, which is updated in the following ways:
\begin{equation}
	x_{i}^{k+1}=\big(p_{i}^k+p_c^k\bigr)(1+\alpha_f\gamma)/2
\end{equation}

(3) Mechanisms of lion cub evolution
Young lions have three renewal stages: co-evolution with the king, co-evolution with the lioness, and as adults, being driven out of their territory, being trained and trying to move towards the best position in their memory, 
\begin{equation}
	x_i^{k+1}=\left\{\begin{aligned}
		&(g^k+p_i^k)(1+\alpha_c\gamma)/2,&&q\leq1/3\\
		&(p_m^k+p_i^k)(1+\alpha_c\gamma)/2,&&1/3<q\leq2/3\\
		&(\bar{g}^k+p_i^k)(1+\alpha_c\gamma)/2,&&2/3<q\leq1
		\end{aligned}
	\right.
\end{equation}

where $g^k$ signifies the optimal population position achieved after the $k^{th}$\ iteration. $\gamma $ is a random number generated according to the standard normal distribution. The historical optimal position of the $i^{th}$\ lion after $k$ iterations is denoted by $p_i^k$. Similarly, $p_c^k$\ and $p_m^k$ represent the historical optimal position of a randomly selected lioness after $k$ iterations. $q$ is a uniform random number within the range of $[0,1]$. The position where a lion cub is expelled from the territory is called $\bar{g}^k$, defined as follows:
\begin{equation}
	\bar{g}^k=\overline{\text{low}}+\overline{\text{high}}-g^k
\end{equation}
where $\overline{\text{low}}$ and $\overline{\text{high}}$ respectively represent the minimum and maximum mean values of each dimension of the population optimization space. $\bar{g}^k$ is the position far away from the lion king, which is an elite reverse learning idea. $\alpha$ is the disturbance factor defined by the original algorithm, which can make the activity range of the lion vary from large to small with the number of iterations increasing, and improve the global exploration ability in the early stage and local development ability in the late stage of the algorithm. The disturbance factors of the lioness and the lion cubs are defined as follows:
\begin{gather}
	\alpha_f=step*\exp{\left[-30\left(\dfrac{t}{T}\right)^{10}\right]}\\
	\alpha_c=step*\left(\frac{T-t}{T}\right)\ 
\end{gather}
where the step is a coefficient representing the lion's maximum disturbance range \cite{liu2018swarm}. Many scholars have proved that LSO has excellent robustness and convergence speed in solving single-objective optimization problems, but it was difficult to achieve such effects by using LSO to deal with multi-objective optimization problems.

\section{The proposed MF-DMOLSO}
LSO lacks a mechanism to escape local optima, leading to a significant reduction in individual disturbance range in the later stage. This results in poor global optimization capabilities. When dealing with multi-objective optimization problems, the selection method of individual historical or global optima cannot be as simple as that of LSO. This is a crucial factor impacting the performance of MOEAs. Its architecture is presented in Fig.1. 
\begin{figure}[H]
	\centering
	\includegraphics[scale=0.34]{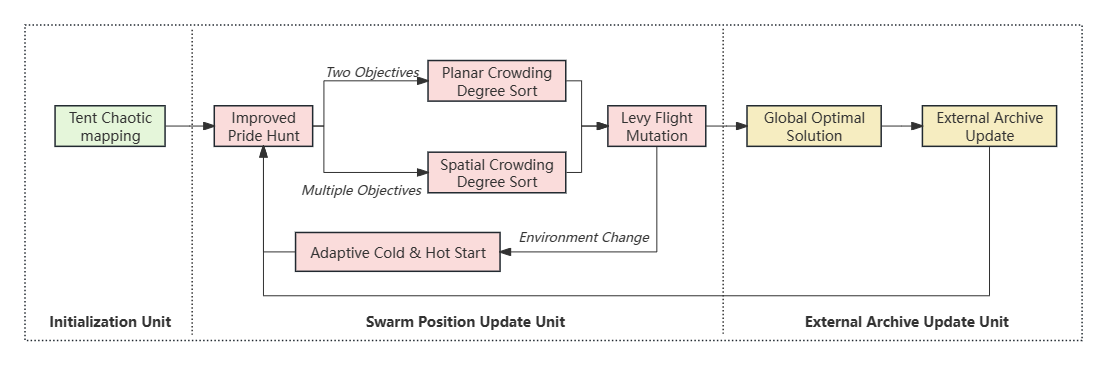}
	\caption{MF-DMOLSO architecture. MF-DMOLSO comprises three units: initialization unit, swarm position update unit, and external archive update unit.}
\end{figure}
\subsection{Initialization unit}

\subsubsection{Tent chaotic mapping}
LSO utilizes a random number generator for population initialization, which may lead to a loss of population diversity \cite{wu2020improved,wu2021lion}, especially when the population size is small. When dealing with complex objective functions and high-dimensional optimization, initialization has a profound impact on subsequent update iterations. To address this, we chose Tent chaotic mapping for population initialization.

Tent chaos exhibits ergodicity, randomness, and long-term unpredictability, characteristics that enhance the global search ability of the population within the search space \cite{zhang2023improved,ming2017multi}. Its mathematical description is as follows:
\begin{equation}
	x_{n+1}=f(x_n)=\left\{
	\begin{aligned}
		&\dfrac{x_n}{\alpha}&x_n\in[0,\alpha)\\
		&\dfrac{(1-x_n)}{(1-\alpha)}&x_n\in[\alpha,1]
	\end{aligned}
	\right.
\end{equation}
where Tent mapping only shows chaotic behavior when $\alpha\in\left(0,1\right)$. To avoid the system entering the periodic state, $\alpha\neq0.5$ and $x_1\neq\alpha$ must be satisfied. Especially when the population is small, random points generated by the rand function often result in missing values in each dimension, leading to significant frequency differences. Conversely, chaotic points generated by Tent mapping exhibit a wider coverage with minimal frequency differences.

\subsubsection{Chaotic mapping initialization}
In practical applications, the parameter thresholds of each dimension varied, necessitating the generation of chaotic sequences for each dimension during initialization. Here is how it could be implemented:
\begin{equation}
	\left\{
	\begin{aligned}
		&x_{n+1}=f(x_n)=\left\{
		\begin{aligned}
			&\dfrac{x_n}{\alpha}&x_n\in[0,\alpha)\\
			&\dfrac{1-x_n}{1-\alpha}&x_n\in[\alpha,1]
		\end{aligned}
		\right.\\
		&X_{n+1}^{dim=i}=U_l^{dim=i}+x_{n+1}\times(U_h^{dim=i}-U_l^{dim=i})
	\end{aligned}
	\right.
\end{equation}
where dim represents a dimension of the parameter and $X_n^{dim=i}$ represents the position of the $n^{th}$ lion in dimension $i$. $U_h^{dim=i}$ and $U_l^{dim=i}$ represent the upper and lower limits of the search space in dimension $i$, respectively.

\subsection{Swarm position update unit}
\subsubsection{Pareto fast non-dominated sort}
For a multi-objective optimization problem $F\left(x\right)=\min(f_1\left(x\right),f_2\left(x\right),\ldots,f_m(x))$ involving $m$ sub-objective functions, let's assume two feasible solutions $x_1$ and $x_2$ that satisfy the constraints and have $F\left(x_1\right)\le F\left(x_2\right)$. If $x_1$ outperforms $x_2$ on at least one subobjective $f_i\left(x\right)(i\in[1,m])$, that is, $f_i\left(x_1\right)<f_i\left(x_2\right)$, then $x_1\succ x_2$. Solutions that are not dominated by any other solutions in the feasible set are referred to as non-inferior or Pareto optimal solutions, and the set containing them is called a non-inferior or Pareto optimal solution set \cite{coello2004handling}. The individuals in the population are assigned Pareto levels. The Pareto-1 level consists of the optimal solution set, and the Pareto-2 level contains the next set of optimal solutions found among the remaining individuals based on the principles outlined above. This process is repeated until each individual is assigned a Pareto level, ensuring that individuals at the same level do not dominate each other. For specific steps, refer to the non-dominated sorting pseudo-code in Table 1.
\begin{table}[htbp]
	\caption{Pareto non-dominated sorting algorithm}
	\begin{tabular}{l}
		\hline
		\textbf{Pseudo-code: Pareto non-dominated sorting}\\
		\hline
		01: flag=1,\ input $X=\left\{x_1,x_2,\ldots,x_n\right\}$,\\ $F\left(x\right)=\left\{f_1\left(x\right),f_2\left(x\right),\ldots,f_m\left(x\right)\right\}$\\
		02: while(The set X is non-empty)\\
		03: \;\;for $i=1$ to the last individual in $X$\\
		04: \;\;\;\;if $x_i$ isn’t dominated by any other individual $x$ in set $X$\\
		05: \;\;\;\;\;\;Pareto level of $x_i=flag$\\
		06: \;\;end for\\
		07: \;\;Delete the individuals whose Pareto level is flag in $X$ set\\
		08: \;\;flag=flag+1\\
		09: end while\\
		\hline
	\end{tabular}
\end{table}

\subsubsection{Adaptive cold and hot start strategies}
Cold start refers to the way to re-initialize the population with the dynamic change of the environment of the optimization problem, which is not related to the optimization result of the previous time, but only related to the value range of the decision space. Hot start is a way to carry out corresponding evolutionary operations (such as mutation) based on the optimization results of the previous time to produce the initial evolutionary population of the next optimization time, which is not only related to the optimization results of the previous time but also related to the value range of the decision space.

In the dynamic change of environment, there are several cases of strong correlation, weak correlation, and no correlation in the Pareto optimal front before and after the time. If there is a strong correlation between the dynamic changes before and after the environment, hot start can make better use of the previously obtained optimization results and re-initialize the population near the original PF. If the correlation is weak, the population initialization needs to be done more widely in the original search space, so as not to make the search space not sufficiently traversed. Therefore, to better adapt to various uncertain dynamic environments, the advantages of cold and hot starting methods are considered when generating a new initial population, and the combination of cold and hot starting is adopted \cite{huan2015mixture}. 
\begin{figure}[H]
	\centering
	\includegraphics[scale=0.16]{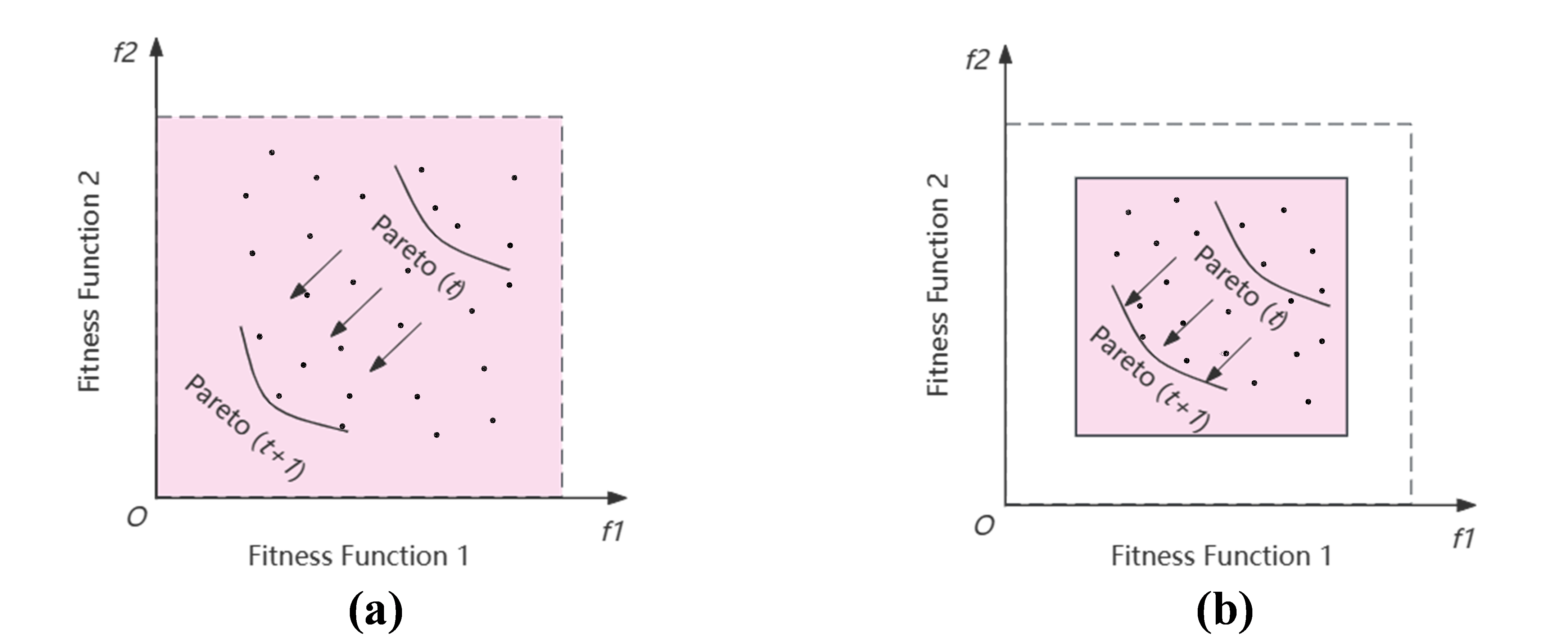}
	\caption{Two start-up methods: (a) weak association – cold start mode, and (b) strong association – hot start mode.}
\end{figure}

In addition, to effectively improve the diversity of individuals and enhance the intensity of evolutionary disturbance under hot start, Cauchy variation will be used to enhance the diversity of the initial population. After the population initialization of cold start and hot start respectively, the two populations are merged, and the non-dominated ranking is carried out in this new population to select individuals who are  better in the new environment until the population is full. The population composed of these better individuals is the result of the adaptive start strategy.

\subsubsection{Improved cub behavior and step formula}
As the most flexible individual, lion cubs are partly responsible for local optimization and partly for global optimization. LSO has a fixed number of lion cubs for different tasks, which lacks flexibility. In this paper, a regulatory factor  is introduced to improve the flexibility of lion cubs' behavior. The improved positional update formula is as follows:
\begin{equation}
	x_i^{k+1}=\left\{\begin{aligned}
		&(g^k+p_i^k)(1+\alpha_c\gamma)/2,&&q\leq\dfrac{\eta}{2}\\
		&(p_m^k+p_i^k)(1+\alpha_c\gamma)/2,&&\dfrac{\eta}{2}<q\leq\eta\\
		&(\bar{g}^k+p_i^k)(1+\alpha_c\gamma)/2,&&\eta<q\leq1
	\end{aligned}
	\right.
\end{equation}
where  $\eta=\dfrac{2t+T}{3T}$, it changes dynamically over time. In the early phase of the algorithm, a larger search range is essential to enable the lion cubs to conduct a more extensive global search. This entails increasing the number of lion cubs implementing the reverse learning strategy and broadening the search scope, thereby enabling the population to quickly identify the actual PFs. Subsequently, in the later phase of the algorithm, the lion cubs focus on more local searches to refine their exploration of the complete PFs and achieve a more precise convergence. This approach enhances both the efficiency and robustness of the algorithm.

The step length formula for lion cubs in LSO is presented in Equation (6). However, the step size decline rate of this formula is fixed, resulting in a large step size in the middle and late stages. This approach often leads to insufficient precision and individuals being close to, but unable to reach, the global optimal. To address this issue, we proposed an enhanced step length formula for lion cubs to improve their optimization capabilities:
\begin{equation}
	\alpha_c=\left\{
	\begin{aligned}
		&step*\dfrac{T-t}{T},&0\leq r\leq 0.7\\
		&step*\exp\left[-30\left(\dfrac{t}{T}\right)^{10}\right],&0.7<r\leq1
	\end{aligned}
	\right.
\end{equation}
where, $r$ is the number generated randomly in [0,1], when $0< r\leq0.7$, the original step length is used to search, when $0.7<r\le1$, the same step length as the lioness is used to search, to ensure that some lion cubs’ step length declines in the early to search sufficiently, and to prevent excessive step length in the late for more accurate search.

\subsubsection{ Levy flight mutation based on crowding degree sorting}
LSO or NSGA-II \cite{deb2002fast} tends to converge prematurely during the optimization process, making it difficult to escape local optimal solutions. Consequently, a significant number of individuals cluster around these local optima. To address this issue, this paper introduces a crowding degree sorting operation. Initially, the population is sorted based on the values of a specific objective function. The crowding distance $C$ for each individual is then determined by calculating the Euclidean distance between two adjacent individuals in the D-dimensional search space, as defined below:
\begin{equation}
	C_i=\|X_{i+1}-x_{i-1}\|_{2}={\sqrt{\sum_{j=1}^{\mathrm{D}}(x_{i+1,j}-x_{i-1,\mathrm{j}})^{2}}}
\end{equation}

Levy flight, a random search strategy, involves both long and short-distance movements—a characteristic that aligns with the foraging habits of fruit flies and other organisms. Since the motion step size follows the Levy distribution, it can make individuals jump out of the local optimal with a greater probability and is therefore widely used in the improvement of optimization algorithms \cite{guan2019modified}. For instance, Du et al. \cite{du2017adaptive} introduced Levy flight to enhance the traditional particle swarm optimization algorithm.

Incorporating an elite learning strategy based on Levy flight, this paper enables individuals to traverse long or short distances towards external solution repositories, thereby expediting convergence, as per formula (12).
\begin{equation}
	x_i^{k+1}=\left\{
	\begin{aligned}
		&\lambda|L|(g^k-x_i^k)+x_i^k,&r\leq0.7\\
		&\lambda|L|(x_j-x_i^k)+x_i^k,&r>0.7,\; x_j\in \{X|X \; \text{is\ in F(1) of external archive}\}
	\end{aligned}
	\right.
\end{equation}
In the above equation, $r$ is a random number,
\begin{equation}
	L=\dfrac{\mu}{|v|^{\dfrac{1}{\beta }}}
\end{equation}

Among them, $\mu\sim N\left(0,\sigma_\mu^2\right)$, $v\sim N\left(0,\sigma_v^2\right)$, where $\sigma_v=1, \;\sigma_\mu $ are described mathematically as follows: 
\begin{equation}
	\sigma_\mu=\left\{\dfrac{\Gamma(1+\beta)\sin(\dfrac{\beta\pi}{2})}{\Gamma(\dfrac{1+\beta}{2})2^{\frac{\beta-1}{2}}\beta^{}}\right\}^{\dfrac{1}{\beta}}
\end{equation}
where, $\beta$ generally takes a value of 1.5.

The smaller the crowding distance $C$, the higher the crowding degree. When the crowding distance falls below a certain threshold, a lion swarm individual can be identified as a "clustering" phenomenon within a specific section of the Pareto frontier. This threshold takes the average of the crowding of all individuals, so the determination of crowding depends on the real-time situation of each iteration. In this study, a select group of hunting lions are dispatched from crowded areas to execute Levy flight mutations. First, make sure that individuals in crowded areas do not dominate the global optimal solution or other solutions in F(1) of the external archive. Then individuals in crowded areas are selected at intervals to perform Levy flight mutation. This approach ensures that local development continues within the region while enabling some individuals to explore globally, thereby enhancing population hunting diversity.

Furthermore, distinct flight step ratio coefficients $\lambda$ are assigned to individuals at different Pareto levels. Lower Pareto levels, indicating individuals farther from the actual Pareto frontier, are assigned higher step ratios, enabling a swifter approach to the optimal frontier. The flight directions also vary: some lions target the global optimal position, while others randomly select solutions from the optimal solution set as their flight direction.

\subsubsection{Population selection strategies based on reference points for multiple objectives}
Because non-dominated solutions occupy most of the population, it is difficult for any elite-protected MOEA to accommodate a sufficient number of new solutions in the population. When the proportion of non-dominated solutions increases exponentially with the increase of the target number, the search process is greatly slowed down. Secondly, implementing diversity preservation, such as crowding, would be a computationally time-consuming operation, and the crowding distance method does not perform well in balancing the diversity and convergence of algorithms. This is also a common problem of MOEAs. Therefore, in the face of multi-objective or many-objective, this paper introduces the reference point mechanism to improve the performance of MF-DMOLSO.

To solve the above problems, Deb et al. [9] first proposed NSGA-III. As an improved algorithm of NSGA-II, NSGA-III is a decomposition-based algorithm, which includes a niche based on reference points and can adaptively build a spatial hyperplane to make the population converge in a better direction. In this paper, the reference point mechanism is introduced into MF-DMOLSO to improve the global convergence speed of the algorithm in the high-dimensional target space, avoid the performance degradation caused by the small Pareto selection pressure, and maintain the uniformity of the solution distribution in the high-dimensional target space.

After generating an offspring population ($N$ individuals) from the parent population ($N$ individuals), we merge them into a combined population $S(t)$ (2N individuals) for non-dominated sorting. Starting from the domination layer with Pareto-1 level, individuals are iteratively added to the next generation population $S(t+1)$ until it reaches the size of $N$. If the $N^{th}$ individual added to $S(t+1)$ belongs to a non-domination layer F(l) where there are remaining individuals who cannot be accommodated, the reference point mechanism is employed to reselect the most suitable individuals from F(l) for inclusion in $S(t+1)$. Here's the technical procedure to accomplish the task \cite{geng2020improved}:

(1) Reference point generation is based on Das and Dennis' method, which is the most commonly used systematic method for evenly distributing reference point sampling over the unit simplex, in the following form:
\begin{equation}
	S_j\in\left\{\dfrac{0}{H},\dfrac{1}{H}\ldots,\dfrac{H}{H}\right\},\, \sum_{j=1}^{M}{S_j=1}
\end{equation}
where $H$ is the number of target shards and M is the number of targets.

(2) Performing adaptive normalization of the population. Firstly, the ideal point of the population $S_{grade}^t$ is defined as the minimum value $z_i^{\min},i=1,\ 2,\ldots,\ M$, across all historical populations ${U_{\tau|=0}^{t}S}_{grade}^\tau$. This constructs the ideal point $\bar{z}=(z_1^{\min},z_2^{\min},\ldots,z_M^{\min})$. Based on this ideal point, the objective functions are transformed as
\begin{equation}
	f_i'\left(x\right)=f_i\left(x\right)-z_i^{\min}
\end{equation}
Secondly, the extreme points corresponding to each coordinate axis are determined by evaluating
\begin{equation}
	ASF\left(x,w\right)=\max_{i=1}^{M_{grade}}\dfrac{f_i'(x)}{w_i},\ x\in S_{grade}^t
\end{equation}

The intercepts $a_i,\ i=1,\ 2,\ldots,\ M$, which are ultimately sought in our solution, are represented by the intersections between the hyperplane formed by M extreme points and the coordinate axes in the M-dimensional space. Subsequently, the objective function values are adaptively normalized according to the following formula
\begin{equation}
	f_i^n\left(x\right)=\dfrac{f_i'\left(x\right)}{\left(a_i-z_i^{\min}\right)}=\dfrac{\left(f_i\left(x\right)-z_i^{\min}\right)}{\left(a_i-z_i^{\min}\right)}
\end{equation}

(3) Individual association and retention. The reference line $d\left(s_{ref},w_{ref}\right)$ of all individual $s_{ref}$ from the reference point $w_{ref}$ is found, and the nearest reference point $\pi\left(s_{ref}\right)$ and its corresponding distance $d\left(s_{ref}\right)$ are obtained. The reference point $\pi\left(s_{ref}\right)$ closest to the population individual is considered to be associated with the individual.
\begin{gather}
	d\left(s_{ref},w_{ref}\right)=\ \dfrac{\left(s_{ref}-{w_{ref}}^T\cdot s_{ref}\cdot w_{ref}\right)}{\left|w_{ref}\right|}\\
	\pi\left(s_{ref}\right)=w_{ref}: argmin\ d\left(s_{ref},w_{ref}\right)\\
	d\left(s_{ref}\right)=\ d\left(s_{ref},\pi\left(s_{ref}\right)\right)
\end{gather}

There are two situations: one is that a reference point is associated with one or more individuals; second, there are no individuals associated with it. As specifically shown in Fig.3.
\begin{figure}[H]
	\centering
	\includegraphics[scale=0.1]{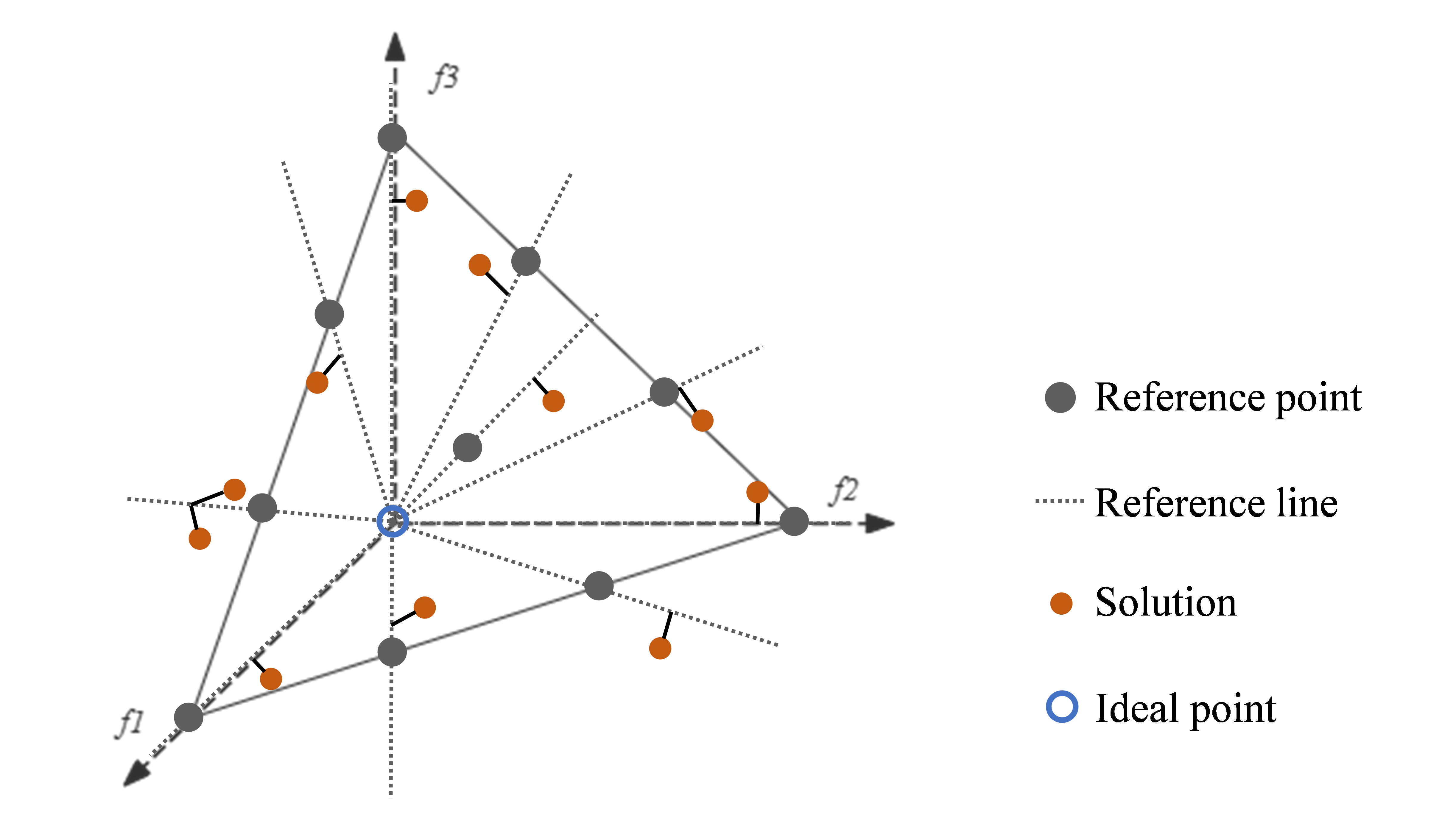}
	\caption{Reference point mechanism. We generally show how population individuals relate to reference points.}
\end{figure}

The principle in the algorithm is that individuals corresponding to less connected reference points should be retained to maintain diversity. Therefore, the reference point with the least number of associations is first selected. If there are multiple such reference points, one $\overline{j}$ is randomly selected from them, and $I_{\overline{j}}$ is the individual associated with the reference point $ \overline{j}$. If $I_{\overline{j}}$ is empty, the reference point is reselected. Otherwise, it depends on whether $\rho_{\overline{j}}$ is 0. $\rho_{\overline{j}}$ is the number of individuals associated with the reference point $\overline{j}$ in all non-dominated layers before the F(l) layer. If $\rho_{\overline{j}}=0$, the individual closest to $\overline{j}$ is selected from $I_{\overline{j}}$ to enter the next generation. Otherwise, a random individual from $I_{\overline{j}}$ is selected to enter the next generation. This operation is repeated until the size of $S(t+1)$ is $N$.

\subsection{External archive update unit}
\subsubsection{Optimal location selection mechanism}
The globally optimal individual significantly influences the flight mutation and updating of the entire lion swarm, serving as a crucial factor in enhancing the overall algorithm performance. Typically, the globally optimal individual steers the population towards the true Pareto front or distributes diversity along the Pareto front. When the Pareto selection pressure is insufficient, the convergence of the selected individuals may not be excellent enough to guide the population convergence and hinder the effective convergence process. On the other hand, poor diversity can easily result in premature convergence of the population, leading to the loss of diversity. Therefore, when devising strategies for selecting the globally optimal individual, it is essential to strike a balance between convergence and distribution, thereby effectively promoting the evolutionary process.

For a two-objective problem, when the solutions on PF are tightly clustered, the sparse distribution regions require a larger number of lion individuals to effectively search for optimal development. Therefore, we select the least crowded point on layer F(1) of the external file as the global optimal position for guiding the evolution of individuals in subsequent iterations. 
\begin{figure}[H]
	\centering
	\includegraphics[scale=0.12]{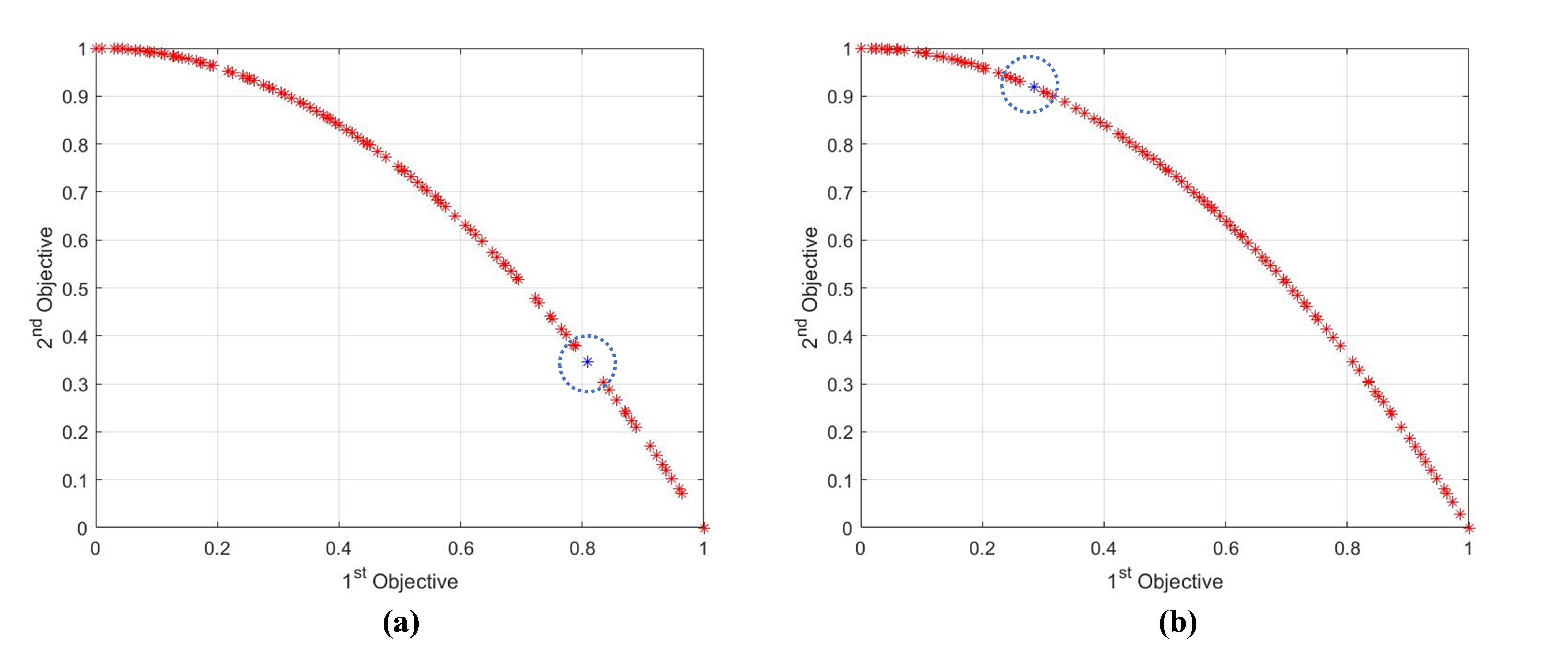}
	\caption{Optimal location selection and external archive update results based on congestion degree: (a) Optimization result of a certain iteration, and (b) Optimization result after a few more iterations. The dot circled represented the global optimal position that lay within the sparsely distributed region of the PF's optimal solution.}
\end{figure}

As seen in Fig.4 (a) and (b), after several iterations, the solutions in the original sparse interval had become more abundant, while the regions of dense solution concentration had become relatively sparse.

For global optimization in three-dimensional and higher spaces, this paper proposes a global optimal selection strategy based on the crowding degree of associated points. Since the same reference line may be associated with multiple individuals, a higher number indicates a greater density of individuals near that reference line. Therefore, the number of individuals associated with a reference line is used as a diversity metric, denoted as the crowding degree $\varphi_j (j =1,2,\ldots, Q$, where Q is the index of the reference line). A higher value of $\varphi_j$ indicates a denser distribution of individuals near reference line j and poorer diversity. When two non-dominated solutions are associated with the same reference line, they are considered to have the same diversity. To further evaluate the convergence of non-dominated solutions, the distance from a non-dominated solution x to the origin is defined as a convergence metric, calculated as follows:
\begin{equation}
	\sigma_x=\sqrt{\sum_{i=1}^{M}{f_i^n\left(x\right)}^2}
\end{equation}

A smaller $\sigma$ value indicates that a non-dominated solution is closer to the origin, thereby indicating better convergence of the solution. Therefore, near the reference line with the least number of associated individuals, that is, among the non-dominated solution with the least $\varphi_j$, the solution with the smallest $\sigma$ value is selected as the global guide. 

So far, we define different crowding representations in different dimensional spaces. The global optimal solution can be selected. Meanwhile, Levy flight mutation based on crowding degree is still applicable in the high-dimensional objective space. These ensure that in guiding the next round of evolution, diversity is maintained while the acceleration of the optimization process is achieved more efficiently through elitism.

\subsubsection{External archive update mechanism}
In the solution space of multi-objective problems, there is typically more than one optimal solution. Therefore, it is essential to establish an external archive dedicated to storing exceptional solutions and updating it after each population update \cite{deb2002fast, dong2021research}. The algorithm maintains external archive by crowding degrees in different spatial dimensions. Refer to the accompanying pseudo-code for this process.
\begin{table}[htbp]
	\caption{External archive updating mechanism}
	\begin{tabular}{l}
		\hline
		\textbf{Pseudo-code: External archive updating based on crowding degree}\\
		\hline
		01: $X_{new}=X\cup X_{old}$ \\
		02: Pareto sorting in the $X_{new}$, non-dominant layers are divided into $F(1) - F(M)$\\
		03: Crowding degree sorting in the $X_{new}$\\
		04: while(number of $X_{new}$ $>$ upper limit of external archive)\\
		05: \;\;Delete the most ‘crowded’ solutions starting at the $F(M)$ layer\\
		06: end while\\
		07: Note: $X,X_{new},X_{old}$ represent the current generation solution, new and old solution set in external archive\\
		\hline
	\end{tabular}
\end{table}

The new population solutions are added to the external archive created for this study, combined into a new collection $X_{new}$. The solutions in $X_{new}$ are reordered according to Pareto rules, and the crowding degree of all solutions is evaluated. Since the number of $X_{new}$ exceeds the upper limit set by the external file, we delete the solution from the lowest level of Pareto level, and delete the most crowded individual first until the size of $X_{new}$ meets the requirements. The usage of crowding degree metrics to manage PF helps to eliminate redundant solutions and maintain the diversity of solutions \cite{sivasubramani2011multi}. 

\subsection{Working process of MF-DMOLSO}
In summary, MF-DMOLSO proceeds as follows: After the initial chaotic mapping, the population's position is updated. Subsequently, the fitness and crowding degree of each individual is calculated. If the crowding degree is high, a Levy flight mutation is applied to a subset of lions. Then the current population and the solutions stored in the archive are formed into a new set, in which better individuals are selected to form the next-generation population or new archive through non-dominant relationships and diversity assessment. When the external environment shifts, the adaptive cold and hot start strategy is employed to re-initialize the population. The specific process is shown in Fig.5.
\begin{figure}[H]
	\centering
	\includegraphics[scale=0.53]{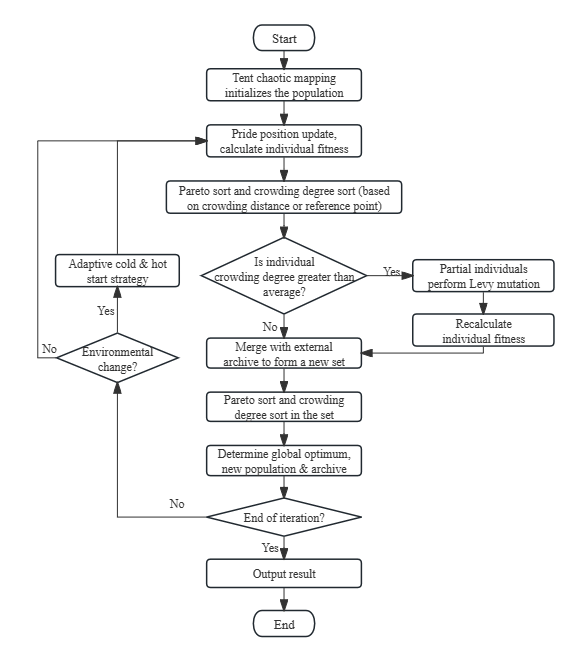}
	\caption{Work process of MF-DMOLSO. Here we show the full flow of the algorithm. It is worth noting that whether the crowding degree sort operation uses the crowding distance sort or the sort based on reference points depends on the target space dimension. Crowded distance sort is used only in two target spaces, and in other cases another way is used.}
\end{figure}
\subsection{Convergence analysis of MF-DMOLSO}
The population sequence ${s(k),k\geq0}$ generated by LSO algorithm is a finite homogeneous Markov chain, where k is the number of iterations, which has been proved in the literature [12], that is to say, the state transition of any lion is only related to the optimal position $g$ of the group, the individual historical optimal position $p_{c}$ of the lioness's partner lion, and the individual optimal position $p_m$ of the lion mother, but not to $k$. And since LSO and multi-objective lion swarm optimization (MOLSO) have the same way of individual evolution, MOLSO also has the above properties. This paper took the evolution of lioness as an example to prove the convergence of MF-DMOLSO algorithm.

\begin{definition}
	If the position of the lion at time k is x(k) and N is any position in the search space, then the position convergence of the particle is defined as follows:
	\begin{equation}
		\lim\limits_{k\to\infty}x(k)=M
	\end{equation}
\end{definition}
\begin{theorem}
	If $\gamma$ satisfies the standard normal distribution N(0,1), then MOLSO method must converge, where, $\gamma$ is the random number in equations (1), (2), (3).
\end{theorem}
\begin{proof}
	In this paper, three kinds of position updating methods were given according to different characteristics of the lion swarm. The updating mode of lioness was selected for convergence verification.
	
	According to Formula (2), the update of the lioness's position is related to the historical optimal position of this lioness and the randomly selected partner lion, and has nothing to do with the time and the position at this time. The historical best location is updated to the current location only if the current location is better than the historical best location. Therefore, the new historical optimal location must be generated from the previously finalized historical optimal location. The details are as follows,
	
	\begin{flalign}
		\begin{aligned}
		x_i^{k+1}&=(1+\alpha_k\gamma)\dfrac{p_i^k+p_c^k}{2}\\
		&=(1+\alpha_k\gamma)\left[(1+\alpha_{k-1}\gamma)\dfrac{p_i^{k-1}+p_m^{k-1}}{2}+(1+\alpha_{k-1}\gamma)\dfrac{p_c^{k-1}+p_n^{k-1}}{2}\right]\dfrac{1}{2}\\
		&=(1+\alpha_k\gamma)(1+\alpha_{k-1}\gamma)\dfrac{(p_i^{k-1}+p_m^{k-1}+p_c^{k-1}+p_n^{k-1})}{2^2}\\
		&\cdots\cdots\\
		&=(1+\alpha_k\gamma)(1+\alpha_{k-1}\gamma)\cdots(1+\alpha_1\gamma)\dfrac{(p_i^{1}+p_m^{1}+\cdots+p_o^{1}+p_q^{1})}{2^k}\\
		\end{aligned}
	\end{flalign}
	where each item in $p_i^1+p_m^1+\ldots+p_o^1+p_q^1$ represents the best position of the first generation per lioness, $2^k$ in total. This historical optimal location may have some duplications, suggesting that this lioness's primary historical optimal location was selected multiple times for location updates. Let the largest item of the $2^k$ items be $p_i^1$. Because it satisfies $\alpha\le \text{step}$, equation (17) satisfies the following inequality,
	\begin{equation}
	\text{Primitive formula}\le\left(1+\text{step}\ast\gamma\right)^k*2^k*p_i^1/2^k=\left(1+\text{step}*\gamma\right)^k* p_i^1
	\end{equation}
\end{proof}
Since $\gamma$ follows the standard normal distribution, when the number of iterations $k$ is sufficiently large, the following formula is valid,
\begin{gather}
	\lim_{k\to\infty}{\left(1+\text{step}*\gamma\right)^k}=1\ \ 
\end{gather}
Thus obtained,
\begin{equation}
	\lim_{k\to\infty} x_i^k\le p_i^1
\end{equation}
Since $p_i^1$ is a bounded constant, this method is convergent.

At the same time, we could find that the final convergence position of the lion swarm algorithm must be within the range of the initial population position, so it could be seen that the uniform distribution of the population in the search space is very important during the population initialization.

\begin{theorem}
	If\;$0<\lambda|L|<2$ , then MF-DMOLSO methods must converge, where $\lambda$ is the step scale coefficient, and L is Levy's random flight length.
\end{theorem}
\begin{proof}
	Theorem 1 proved the convergence of MOLSO, here only need to prove the convergence of Levy flight position update in the improved algorithm.
	Assuming that the global optimal is known, set $G$, the following formula can be obtained by recursion from equation (10),
	\begin{equation}
		\begin{aligned}
			\ x_i^{k+1}=&\lambda L(G-x_i^k)+x_i^k\\
			=&\lambda LG+(1-\lambda L)x_i^k\\
			=&\lambda LG+(1-\lambda L)\left(\lambda LG+\left(1-\lambda L\right){\ x}_i^{k-1}\right)\\
			=&\lambda LG+(1-\lambda L)\lambda LG+\cdots+(1-\lambda L)^{k-1}\lambda LG+(1-\lambda L)^k x_i^1\\
			=&\lambda LG\dfrac{1-\left(1-\lambda L\right)^k}{1-\left(1-\lambda L\right)}+\left(1-\lambda L\right)^kx_i^1\\
			=&G+{\left(1-\lambda L\right)^k[x}_i^1-G]  
		\end{aligned}
	\end{equation}
	When $k$ is sufficiently large, because $0<\lambda |L|<2$ , the limit of the above formula is obtained,
	\begin{equation}
		\lim_{k\to\infty}x_i^k=G
	\end{equation}
	That is, it can eventually converge to the optimal value. Globally, individuals will eventually converge to a Pareto curve. Therefore, MF-DMOLSO method converges.
\end{proof}

	\section{Experimental methods and settings}
\subsection{Benchmark test function}
To evaluate the performance of the proposed algorithm, we selected six static and two dynamic test functions across various parameter dimensions, including ZDT1-3 \cite{hu2014multiobjective}, DTLZ1-3 \cite{gong2009corrections}, G2 \cite{greeff2008solving}, and dMOP1 \cite{chi2019dynamic,liu2020survey}.

The PFs of ZDT1 and ZDT2 have distinct characteristics. ZDT1’s PF is a continuous convex set, while ZDT2’s PF is continuous but non-convex. In contrast, ZDT3's PF is piecewise discontinuous \cite{hu2014multiobjective}. The DTLZ1 problem features multiple local optima, making it challenging for the algorithm to discover the optimal Pareto set. To enhance the difficulty of global optimization, the g function of DTLZ2 was modified to Rastrigin, generating DTLZ3. DTLZ1 offers an opportunity to assess the algorithm's scalability across numerous targets, while DTLZ2 and DTLZ3 evaluate its ability to converge on a true PF \cite{gong2009corrections,deb2002scalable}. The G2 and dMOP1 functions, being dynamic two-objective functions, serve as effective tests of the algorithm’s adaptability to environmental changes.

\subsection{Evaluation indexes}
The evaluation criteria for multi-objective problems are intricate. In this study, we employed three commonly utilized performance indices to evaluate the efficacy of the proposed algorithm. Generation distance (GD) \cite{van1998evolutionary} gauges the distance between the known PF and the true PF. 
\begin{equation}
	GD({\rm PF}_{\text{known}},{\rm PF}_{\text{true}})=\dfrac{\left(\sum_{i=1}^{n}d_i^p\right)^{\dfrac{1}{p}}}{n}
\end{equation}
Where, $n$=number of points in ${\rm PF}_{known}$,\;$d_i=\min\limits_{p\in PF_{true}}\|F(x^i)-F(p)\|,x^i\in PF_{known}$.

Distribution ($\Delta$) \cite{scott1995fault} assesses the spatial distribution of optimal solution individuals within the solution set. 
\begin{equation}
	\Delta=\sqrt{\dfrac{\sum_{i=1}^{n}\left(\bar{d}-d_i\right)^2}{(n-1)}}
\end{equation}
Where, $d_i=\min_{x^j\in S,{x^i\neq x}^j}{\left(\sum_{k=1}^{M}{|f_k\left(x^i\right)-f_k(x^j)|}\right)}$

Error point Rate (ER) represents the proportion of known solutions in the Pareto frontier that fall outside the true frontier. In this paper, when the distance between an individual in the solution set and the true PF was greater than 0.01, the individual was denoted as an error point.
\begin{equation}
	ER=1-\dfrac{\left|S\cap P\right|}{\left|P\right|}
\end{equation}
\subsection{Experimental environment and parameter settings}

The experimental environment consisted of an operating system, Windows 11 (64-bit), a 12th Gen Intel(R) Core(TM) i7-12700H 2.30GHz processor, 16.0GB of RAM, and the simulation platform MATLAB R2021a.

In the static function test session, MF-DMOLSO was compared with MOPSO, NSGA-II and NSGA-III under different dimensions. Each algorithm was tested 20 times. For two-objectives functions, the external archive was set to accommodate 100 solutions, the iteration count was 400, and the population size was 120. In MOPSO, the inertia weight w=0.9, the inertia weight decay rate $w_{damp}=0.55$, the individual learning coefficient $c_1=1$ and the global learning coefficient $c_2=2$. In NSGA-II, the crossover probability was 0.9 and the mutation probability was 0.5. Cross parameter $t_1=20$, variation parameter $t_2=20$. For three-objectives functions, the parameter setting of the algorithm remain unchanged except that the population size was changed to 500 and the size of the external archive is changed to 500. In NSGA-III, the crossover and variation parameters were the same as those of NSGA-II.

In the dynamic function test session, under the same environment variable T, the iteration counts for the algorithms were 2 500, 5 000, 10 000, 15 000, and 20 000 times, respectively.

	\section{Experimental results}
\subsection{Static multi-objective optimization problems}
Table 3 presents the average values of performance indexes for different algorithms on three benchmark test functions. The performance when D=30 and  D=50 is presented in the Appendix.
\begin{table}[htbp]
	\caption{Performance comparison between MF-DMOLSO and other algorithms for 2-objective problems. Boldfaced values represent the best performance achieved.}
	\begin{tabular}{ccccc}
		\hline
		\textbf{Function(D=100)} & \textbf{Performance index} & \textbf{MOPSO} & \textbf{NSGA-II} & \textbf{MF-DMOLSO}\\
		\hline
		ZDT1 & GD & 0.0085 & 0.0845 & \textbf{5.7409E-4}\\
		 	 & $\Delta$ & 0.0080 & 0.0073 & \textbf{0.0049}\\
		 	 & ER & 100\% & 100\% & \textbf{0.9901\%}\\
		ZDT2 & GD & 0.0184 & 0.0988 & \textbf{5.5415E-4}\\
			 & $\Delta$ & 0.0077 & 0.0061 & \textbf{0.0057}\\
			 & ER & 100\% & 100\% & \textbf{25.7426\%}\\
		ZDT3 & GD & 0.0194 & 0.0649 & \textbf{8.8901E-4}\\
			 & $\Delta$ & 0.0102 & \textbf{0.0064} & 0.0070\\
		 	& ER & 100\% & 100\% & \textbf{6.9307\%}\\
		 \hline
	\end{tabular}
\end{table}

Notably, MF-DMOLSO outperformed the other algorithms in most of the indicators. Specifically, in ZDT1, 2, and 3 functions, MF-DMOLSO achieved a GD reduction of 93.2\%, 96.9\%, and 95.4\% compared to MOPSO, meanwhile 99.3\%, 99.4\% and 98.6\% compared to NSGA-II. ER was also lower by 99.0\%, 74.3\%, and 93.1\% than other algorithms. While the average Delta of MF-DMOLSO in ZDT3 was slightly higher than NSGA-II, MF-DMOLSO still outperformed the other algorithms.

In the repeated experiments, it was also found that when the variable dimension was high, MOPSO and NSGA-II had a low success rate of optimization, and the algorithm may converge to a certain local minimum point or a certain local minimum interval.

This paper also quoted the test results of MOLSO and multi-objective lion swarm optimization based on resampling (RMOLSO) proposed by Dong et al. [16] for comparison, as shown in Table 4. It can be seen that MF-DMOLSO proposed in this paper had better optimization performance. While MF-DMOLSO had a similar optimization success rate to RMOLSO, GD could also be reduced by 31.8\% and 24.7\% respectively in ZDT1 and ZDT3.
\begin{table}[htbp]
	\caption{Performance comparison between MF-DMOLSO and other lion swarm algorithms. Boldfaced values represent the best performance achieved}
	\begin{tabular}{ccccc}
		\hline
		\textbf{Function(D=30)} & \textbf{Performance index} & \textbf{MOLSO} & \textbf{RMOLSO} & \textbf{MF-DMOLSO}\\
		\hline
		ZDT1 	& success rate & 60\% & \textbf{90\%} & \textbf{90\%}\\
		& GD & 1.90E-3 & 6.98E-4 & \textbf{4.7608E-4}\\
		ZDT3 	& success rate & 50\% & \textbf{90\%} & \textbf{90\%}\\
		& GD & 8.5531E-4 & 7.2537E-4 & \textbf{5.4627E-4}\\
		\hline
	\end{tabular}
\end{table}

For three-objective test functions, this paper conducted comparative tests on MF-DMOLSO, MOPSO, NSGA-II, and NSGA-III under low dimensional parameters. Table 5 shows the average value of each metric after 20 runs of these algorithms on DTLZ1-3.
\begin{table}[htbp]
	\caption{Performance comparison between MF-DMOLSO and other algorithms for 3-objective problems. Boldfaced values represent the best performance achieved. }
	\begin{tabular}{cccccc}
		\hline
		\textbf{Function(D=10)} & \textbf{Performance index} & \textbf{MOPSO} & \textbf{NSGA-II} & \textbf{NSGA-III} & \textbf{MF-DMOLSO}\\
		\hline
		DTLZ1 & GD & 4.4065 & 1.1290 & 1.2274E-5 & \textbf{9.2529E-6}\\
		& $\Delta$ & 2.7858 & 0.8273 & 3.7142E-4 & \textbf{3.2662E-4}\\
		& ER & 100\% & 100\% & \textbf{0\%} & \textbf{0\%}\\
		DTLZ2 & GD & 7.9161E-4 & 0.0013 & 1.7382E-6 & \textbf{8.0663E-7}\\
		& $\Delta$ & 0.0265 & 0.0278 & \textbf{0.0263} & 0.0295\\
		& ER & 61\% & 89\% & \textbf{0\%} & \textbf{0\%}\\
		DTLZ3 & GD & 21.5033 & 2.2986 & 9.2278E-6 & \textbf{3.0425E-6}\\
		& $\Delta$ & 11.4416 & 0.3849 & \textbf{0.0218} & 0.0279\\
		& ER & 100\% & 100\% & \textbf{0\%} & \textbf{0\%}\\
		\hline
	\end{tabular}
\end{table}

As can be seen from Table 5, MF-DMOLSO was slightly ahead of NSGA-III in most indicators. At the same time, due to the introduction of the reference point mechanism, the optimization accuracy of MF-DMOLSO and NSGA-III algorithms was much higher than that of MOPSO and NSGA-II, indicating that the reference point mechanism plays a very important role in maintaining population diversity in high-dimensional target space. The proposed algorithm, leveraging the inherent advantages of LSO over GA algorithms and incorporating additional refinement strategies, exhibited superior performance compared to NSGA-III in multi-objective optimization problems.

In practice, for the two-objective problem, MF-DMOLSO selects and removes individuals based on the Pareto relationship and crowding distance, preferring solutions with a larger crowding distance on the Pareto non-dominated layer. This approach ensures better population diversity and convergence in two-dimensional space. However, when dealing with multiple objective functions, the crowding distance fails to accurately reflect the discreteness of solutions in the multi-dimensional space. Consequently, the obtained solutions are unevenly distributed on the non-dominated layer, leading the algorithm to converge prematurely into local optima. Theoretically, in optimization problems involving multiple objectives, the algorithm's reliance solely on crowding distance and Levy flight to escape local optima is limited. This limitation is the primary reason why the present study introduced a reference point mechanism for optimization problems beyond two dimensions. Leveraging the inherent advantages of LSO algorithm and the enhanced accuracy of population diversity assessment through the reference point mechanism, MF-DMOLSO demonstrated superior performance on DTLZ functions compared to NSGA-III and other benchmark algorithms.

\subsection{Dynamic multi-objective optimization problems}
In this paper, the optimization effect of MF-DMOLSO in dynamic function was tested. Fig.6 shows the optimization of MF-DMOLSO on dMOP1 function when the evaluation times are 10000 (T=0,1,2,3). Cells in Table 6 are the average value of the evaluation indicators.
\begin{figure}[H]
	\centering
	\includegraphics[scale=0.15]{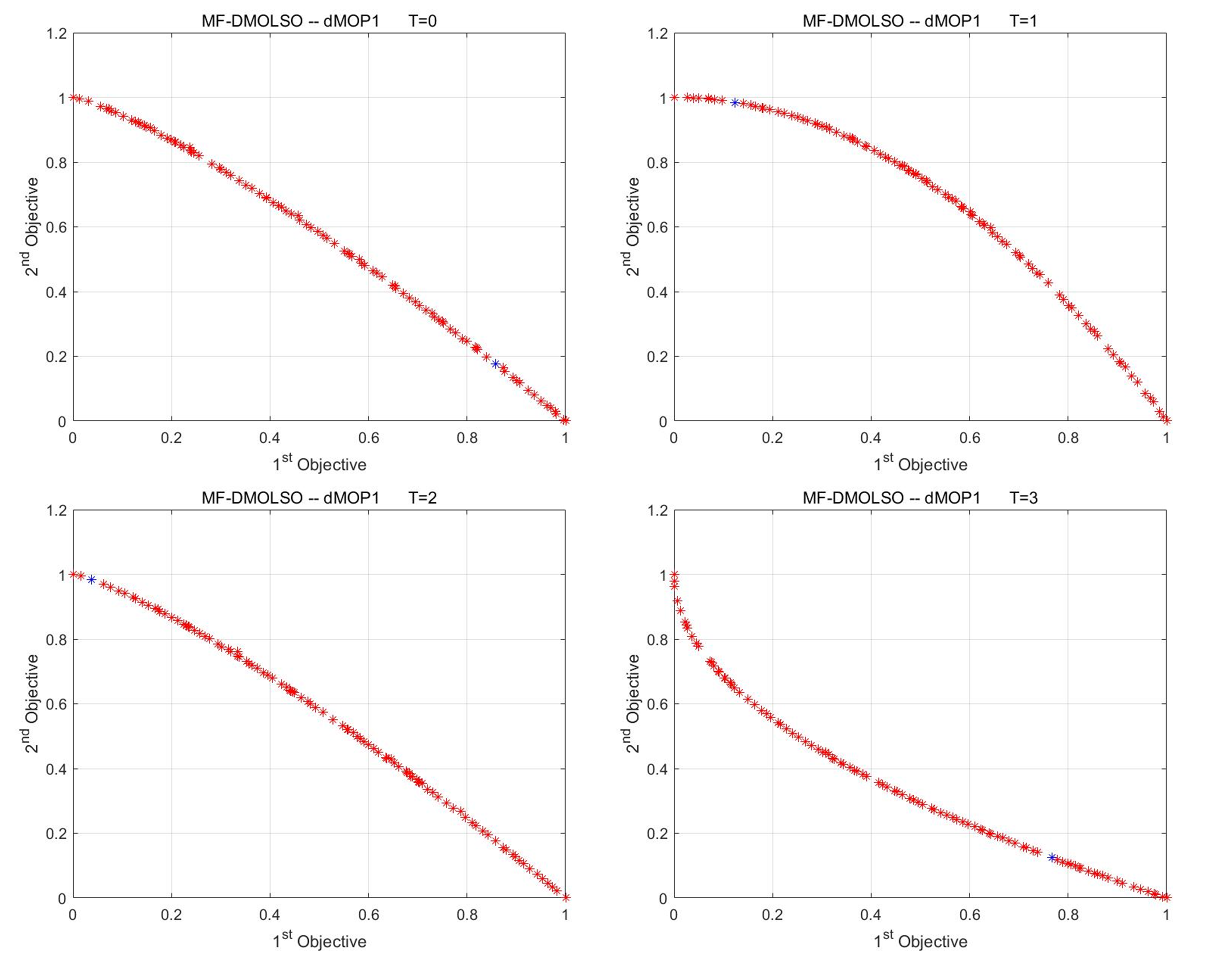}
	\caption{MF-DMOLSO – dMOP1 optimization result. From top left to bottom right, we show the optimization results of MF-DMOLSO with dMOP1’s time variable T from 0 to 3.}
\end{figure}

\begin{table}[htbp]
	\caption{Performance comparison between MF-DMOLSO and other algorithms for 3-objective problems. Boldfaced values represent the best performance achieved. }
	\begin{tabular}{cccccc}
		\hline
		\textbf{Iterations} & \textbf{Index} & \textbf{T=0} & \textbf{T=1} & \textbf{T=2} & \textbf{T=3}\\
		\hline
		2500 & GD & 2.4654E-3 & 8.9541E-4 & 6.4513E-4 & 7.8326E-4\\
		& $\Delta$ & 0.1632 & 0.0098 & 0.0081 & 0.0083\\
		& ER & 94\% & 56\% & 58\% & 62\%\\
		5000 & GD & 1.2004E-3 & 5.5354E-4 & 5.8498E-4 & 5.9868E-4\\
		& $\Delta$ & 0.0931 & 0.0431 & 0.0430  & 0.0432\\
		& ER & 46\% & 41\% & 40\% & 43\%\\
		10000 & GD & 5.7503E-5 & 5.6325E-5 & 5.6366E-5  & 5.9859E-5\\
		& $\Delta$ & 0.0061 & 0.0052 & 0.0057 & 0.0059\\
		& ER & 3\% & 2\% & 2\% & 2\%\\
		15000 & GD & 1.2331E-5 & 1.9559E-5 & 1.2331E-5 & 2.2331E-5\\
		& $\Delta$ & 0.0053 & 0.0051 & 0.0054 & 0.0055\\
		& ER & 0\% & 0\% & 0\% & 0\%\\
		20000 & GD & 1.7216E-5 & 1.7508E-5 & 1.2216E-5 & 1.7205E-5\\
		& $\Delta$ & 0.0054 & 0.0053 & 0.0054 & 0.0054\\
		& ER & 0\% & 0\% & 0\% & 0\%\\
		\hline
	\end{tabular}
\end{table}

In this paper, MF-DMOLSO was also compared with Mixture Crossover Dynamic Constrained Multi-objective Evolutionary Algorithm Based on Self-Adaptive Start-Up Strategy (MC-DCMOEA), which is the first algorithm to propose the self-adaptive cold and hot start-up strategy. It was found that on the ZDT functions when the evaluation times were more than 50 times, the accuracy and distribution of the solution obtained by the proposed algorithm were slightly better than that of MC-DCMOEA. Therefore, this paper would not make a specific comparison. However, it is necessary to emphasize that the algorithm in this paper had a shorter running time and higher optimization efficiency. When evaluated 10,000 times, MC-DCMOEA's program (provided by the original author) ran in 8.89s, while MF-DMOLSO ran in 1.06s. As can be seen, the algorithm proposed in this paper improved its efficiency by 88\% with better real-time performance and can better deal with real-time optimization scenarios of dynamic problems.

	\section{Application of MF-DMOLSO for trajectory planning of 6R robot}
In the field of industrial robotics, enhancing speed offers significant potential to improve production efficiency. Therefore, research into trajectory planning with a focus on time optimization has been extensive. Typically, manufacturers' maximum speed limits serve as constraints. While these settings prevent most safety issues, pure time orientation can lead to the maximum acceleration limits frequently, resulting in full motor loads, which will increase mechanical dither and wear \cite{cai2013robust}. Running at high acceleration is potentially dangerous for robot operators. Additionally, higher accelerations on the same path contribute to increased energy consumption, which is incompatible with energy conservation and environmental sustainability goals. This study presented Pareto-optimal solutions that optimize both time and maximum acceleration during operation to improve efficiency and make the robot's jump curve smoother and the movement more stable and safe. 

\subsection{Robot kinematics modeling and cubic-quintic-cubic polynomial trajectory planning}
In this study, the robot was required to traverse four designated spatial positions while executing specific gestures. Therefore, the trajectory planning approach was developed to ensure both efficient and safe motion. The Cartesian spatial positions and the specified postures at these four points are detailed in Table B.1 of the Appendix B. For experimental validation, the STEP SR-1400 welding robot was employed. Its standard Denavit-Hartenberg (D-H) parameter table is presented in Table B.2.

We employed a cubic-quintic-cubic three-segment polynomial for trajectory planning, ensuring continuous trajectory, velocity, and acceleration. This approach mitigates operational complexity and potential joint position overruns \cite{su2022}. The joint trajectory equation incorporates $a_{ij}$ as the $j^{th}$ parameter on the curve of segment i. The movement times of the three trajectories are\ $t_1\sim t_3$, and $p_1\sim p_3 $ represent the joint angles of these paths. Each $p$ is a vector containing six-axis angles: $p_i=[p_{i1};p_{i2};p_{i3};p_{i4};p_{i5};p_{i6}] (i=1,2,3)$. The joint angle function (22) was derived from the cubic-quintic-cubic three-segment polynomial. The first derivative of the joint angle provided the joint velocity, as per formula (23). The first derivative of the joint velocity corresponded to the joint acceleration, as detailed in the formula (24).
\begin{flalign}
	&\begin{cases}
		p_1=a_{13}t_1^3+a_{12}t_1^2+a_{11}t_1+a_{10}\\
		p_2={a_{25}t_2^5+a_{24}t_2^4+a}_{23}t_2^3+a_{22}t_2^2+a_{21}t_2+a_{20}\\
		p_3=a_{33}t_3^3+a_{32}t_3^2+a_{31}t_3+a_{30}
	\end{cases}\\
	&\begin{cases}
		{\dot{p}}_1={3a}_{13}t_1^2+{2a}_{12}t_1+a_{11}\\
		{\dot{p}}_2={5a_{25}t_2^4+{4a}_{24}t_2^3+3a}_{23}t_2^2+{2a}_{22}t_2+a_{21}\\
		{\dot{p}}_3={3a}_{33}t_3^2+{2a}_{32}t_3+a_{31}
	\end{cases}\\
	&\begin{cases}
		{\ddot{p}}_1={6a}_{13}t_1+{2a}_{12}\\
		{\ddot{p}}_2={20a_{25}t_2^3+{12a}_{24}t_2^2+6a}_{23}t_2+{2a}_{22}\\
		{\ddot{p}}_3={6a}_{33}t_3+{2a}_{32}
	\end{cases}
\end{flalign}

To ensure smooth robotic motion without jitter, the joint movement must be subject to the following constraints: the trajectory commences at the initial target point $P_1$ and terminates at the fourth target point $P_4$. The speed and acceleration at both the starting and ending points must be zero. Additionally, the position, speed, and acceleration of the intermediate path must exhibit continuous behavior within specified limits. If the maximum values for $t_1\sim t_3$ are $t_{\max1}\sim t_{\max3}$, the corresponding equations could be derived.
\begin{flalign}
	&\begin{cases}
		P_1=a_{10}\\
		P_2=a_{13}t_{\max1}^3+a_{12}t_{\max1}^2+a_{11}t_{\max1}+a_{10}\\
		P_2=a_{20}\\
		P_3={a_{25}t_{\max2}^5+a_{24}t_{\max2}^4+a}_{23}t_{\max2}^3+a_{22}t_{\max2}^2+a_{21}t_{\max2}+a_{20}\\
		P_3=a_{30}\\
		P_4=a_{33}t_{\max3}^3+a_{32}t_{\max3}^2+a_{31}t_{\max3}+a_{30}
	\end{cases}\\
	&\begin{cases}
		a_{11}=0\\
		{3a}_{13}t_{\max1}^2+{2a}_{12}t_{\max1}+a_{11}=a_{21}\\
		{5a_{25}t_{\max2}^4+{4a}_{24}t_{\max2}^3+3a}_{23}t_{\max2}^2+{2a}_{22}t_{\max2}+a_{21}=a_{31}\\
		{3a}_{33}t_{\max3}^2+{2a}_{32}t_{\max3}+a_{31}=0
	\end{cases}\\
	&\begin{cases}
		{2a}_{12}=0\\
		{6a}_{13}t_{\max1}+{2a}_{12}={2a}_{22}\\
		{20a_{25}t_{\max2}^3+{12a}_{24}t_{\max2}^2+6a}_{23}t_{\max2}+{2a}_{22}={2a}_{32}\\
		{6a}_{33}t_{\max3}+{2a}_{32}=0
	\end{cases}
\end{flalign}

The constrained equations outlined above are formulated as $Q=a\times T$, where Q represents the vector containing the angle values of waypoints $P_1\sim P_4$, a contains the motion parameters of the curve parameter $a_{ij}$, and T is the square matrix containing the motion times $t_1\sim t_3$. This formula was subsequently transformed into $a=Q\times T^{-1}$, enabling the coefficient of the trajectory equation $a_{ij}$ to be solved, which represented the optimal trajectory parameter. In this study, we opted to utilize MF-DMOLSO in the three-dimensional space of $t_{max1}\sim t_{max3}$ instead of the parameter space containing \, ${a}_{10}\sim a_{33}$. This approach mitigated the computational complexity associated with solving high-order polynomial equations, thereby enhancing both the optimization time and accuracy.

\subsection{Objective function and optimization results}
In this paper, two fitness functions were set as follows,
\begin{equation}
	\begin{cases}
		\text{fitness1}=\min{(t_1+t_2+t_3)}\\
		\text{fitness2}=\min(\max({\ddot{p}}_1,{\ddot{p}}_2,{\ddot{p}}_3))
	\end{cases}
\end{equation}
where, $t_1,t_2,t_3\in(0,3.5]$. The parameter configuration in the optimization method was akin to that of the benchmark function detailed earlier. The joint speed of the robot must not exceed 100°/s, and the joint acceleration must not exceed 60\si{\degree\per\s^2}. The population size was set to 400, and the size of the external archive was set to 200. Fig.7 exhibited all the nondominated solutions obtained in the external archive by MF-DMOLSO and MOPSO during the first 200 iteration steps.
\begin{figure}[H]
	\centering
	\includegraphics[scale=0.141]{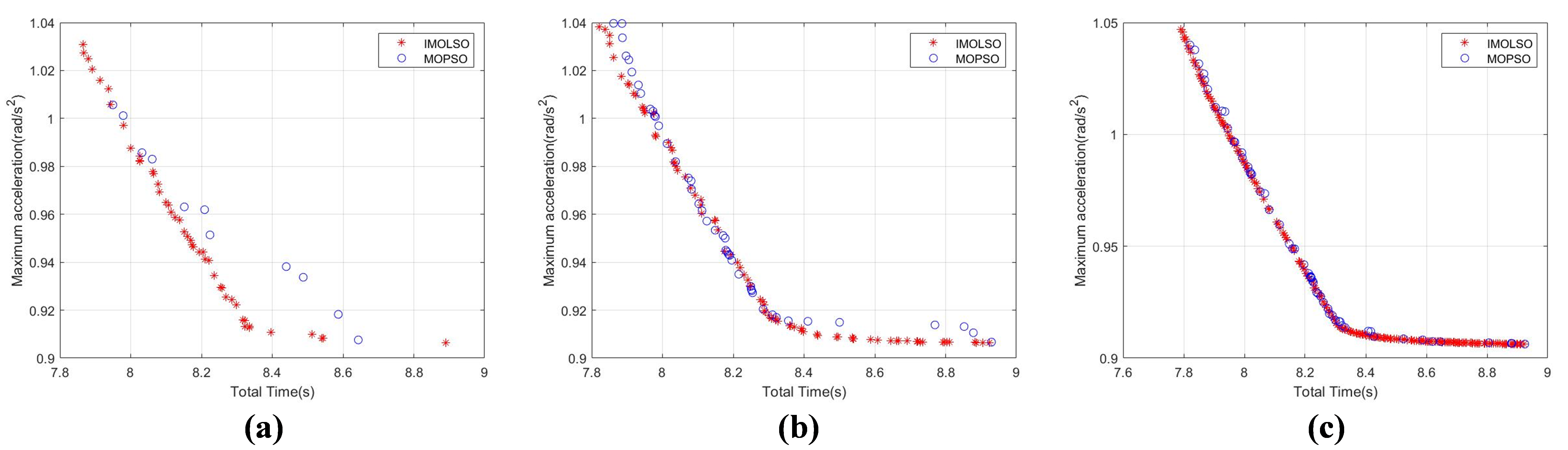}
	\caption{ Comparison of optimization results: (a), (b), and (c) represent the non-dominated solution sets obtained at the 50th, 100th, and 200th generations, respectively.}
\end{figure}

Although both algorithms ultimately identified the true PF due to the relative simplicity of the optimization problem, it was evident that MF-DMOLSO exhibited a faster convergence speed, generated a more diverse solution set, offered higher precision, and achieved a more uniform distribution of solutions. We used Set Coverage ($C$) as a metric to evaluate the dominance relationship between the Pareto solution sets obtained by the MF-DMOLSO and the contrasting algorithm.  Assuming $A$ and $B$ are two PFs, the $C$ value can be expressed as follows, 
\begin{equation}
	C\left(A,B\right)=\dfrac{\left|\left\{b\in B\middle|\exists a\in A:a\succ b\right\}\right|}{\left|B\right|}
\end{equation}
Here, $|B|$ represents the number of solutions in $B$, and $C(A,B)$ indicates the percentage of solutions in $B$ that are dominated by at least one solution in $A$. A higher value of $C(A, B)$ indicates better performance of $A$.
\begin{table}[htbp]
	\caption{Performance comparison between MF-DMOLSO and other algorithms for 3-objective problems. Boldfaced values represent the best performance achieved. }
	\begin{tabular}{ccc}
		\hline
		\textbf{Iteration} & \textbf{\textit{C(MF-DMOLSO,MOPSO)}} & \textbf{\textit{C(MOPSO,MF-DMOLSO)}}\\
		\hline
		50 & \textbf{90.90}\% & 0\%\\
		100 & \textbf{60.42}\% & 9.20\%\\
		200 & \textbf{70.97}\% & 2\%\\
		\hline
	\end{tabular}
\end{table}

The data presented above indicated that the Pareto non-dominated solution set obtained by MF-DMOLSO was able to cover a significant majority of that produced by MOPSO. MF-DMOLSO excelled in optimizing two-dimensional targets with greater efficiency and precision.

\subsection{Optimization decision based on optimal PF and trajectory curve}
Given the rapid decline in the second target value before the total time reaching 8.35s, and the subsequent inability to significantly improve the second objective despite extended timeframes, the trajectory parameter corresponding to (8.35s,0.91\si{\radian\per\s^2}) was selected. This parameter was then applied to the trajectory equation. This study employed Matlab's Robotics Toolbox for robot modeling and trajectory simulation \cite{lu2017}. 
\begin{figure}[H]
	\centering
	\includegraphics[scale=0.121]{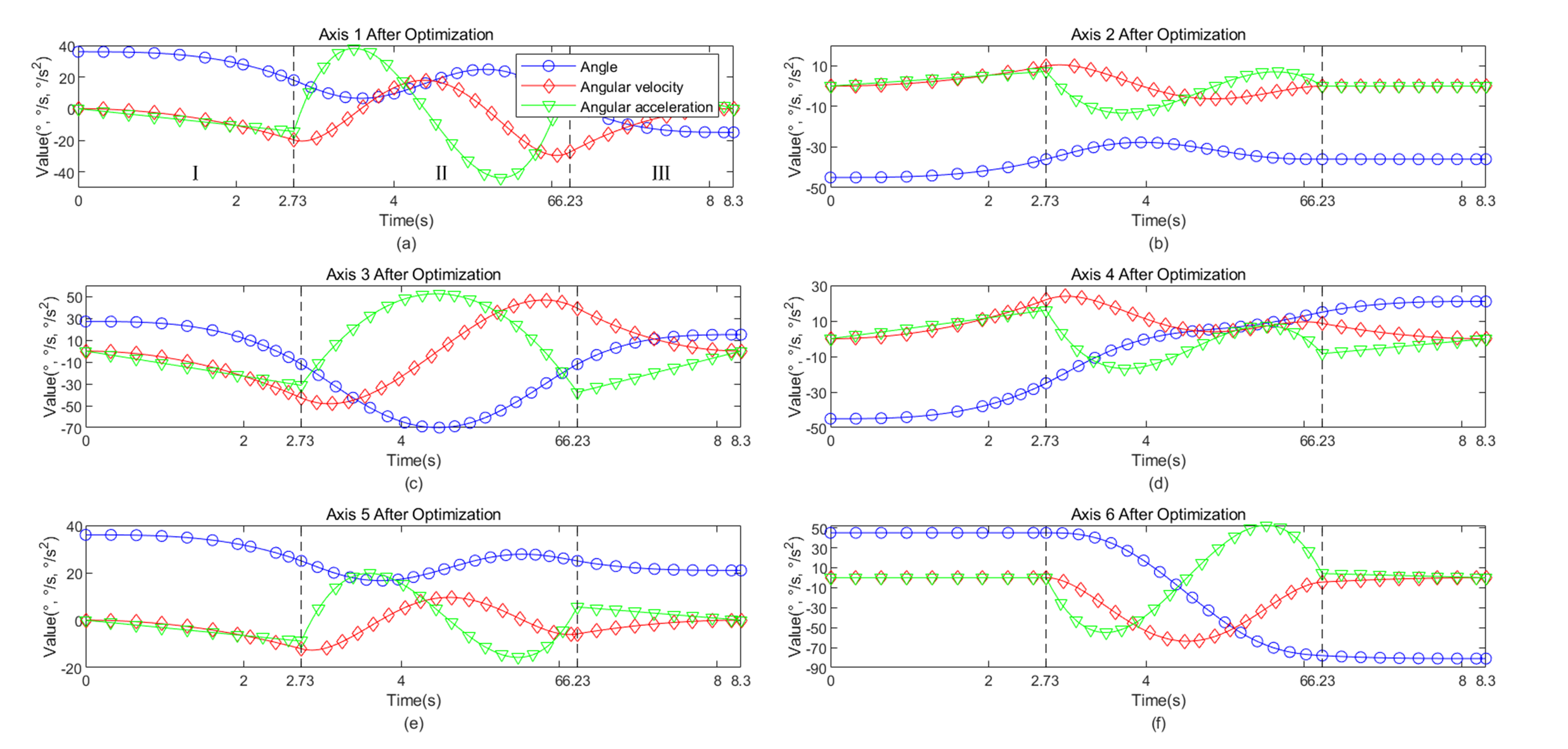}
	\caption{Simulation by Robotics Toolbox. The three trajectories represented by I-III corresponded to cubic-quintic-cubic polynomial equations. Notably, the optimized trajectory not only ensured that the robot reached the specified positions and attitudes at the four target points, but also improved production efficiency by minimizing running time and maximum acceleration.}
\end{figure}

It was evident that the optimized trajectory significantly reduced the running time to 8.3s, and the maximum acceleration was reduced to 54\si{\degree\per\s\squared}. This optimization led to a reduction in mechanical dither, and was beneficial for overall system efficiency.

	\section{Discussion}
In this study, the MF-DMOLSO algorithm exhibited a high success rate exceeding 90\%. When addressing complex high-dimensional functions or intricate Pareto boundaries, MF-DMOLSO consistently outperformed MOPSO and NSGA-II, which showed lower success rates and a tendency to converge towards pseudo-Pareto boundaries, resulting in an ER value of 100\%. Although increasing the population size and iterations marginally improved their performance, these algorithms remained inferior to MF-DMOLSO. Notably, MF-DMOLSO excelled in dynamic optimization problems, delivering superior optimization performance and efficiency. Overall, MF-DMOLSO surpassed other algorithms in approaching the true Pareto boundary, solution distribution uniformity, optimization error rate, success rate, and efficiency.

Enhancements to the lion cub behavior and step formula, along with the incorporation of Levy flight mutation, significantly increased the algorithm's flexibility. The introduction of adaptive cold and hot start strategies provided the algorithm with environmental change awareness, enabling the effective integration of historical information to generate new populations, thereby enhancing dynamic optimization efficiency.

To improve population diversity assessment, this study introduced the crowding degree metric, which varied across dimensions. For two-dimensional space, the crowding degree depended on plane distance, while in three-dimensional space, it relied on the degree of connections between individuals and predefined reference points. In two-dimensional optimization, the algorithm effectively selected and optimized individuals based on crowding distance, resulting in high population diversity and convergence. However, in higher-dimensional objectives, capturing the discreteness of solutions in multidimensional space became challenging, potentially leading to a loss of diversity or uneven distribution within the non-dominant layer. The crowding distance metric became impractical for high-dimensional target optimization, which may lead to a loss of diversity in population evolution \cite{yang2023many}.

Furthermore, in high-dimensional objectives, Pareto non-dominant ordering may not accurately reflect the true quality of solutions, as solutions close to the Pareto boundary could be classified similarly to those farther away. This reduced selection pressure could slow down the convergence rate or trap the algorithm in local optima, particularly with an increasing number of objectives. Specifically, relying solely on Pareto dominance relationships in high-dimensional optimization led to insufficient selection pressure. This inadequacy made it challenging to select a relatively better solution between two individuals with similar advantages and disadvantages. As the number of objective functions increases, the selection pressure may even disappear  \cite{wang2023performance}. Therefore, this study designed different diversity evaluation mechanisms for two-objective and multi-objective problems, effectively addressing optimization issues in various dimensional spaces and balancing solution convergence and distribution.

In two-dimensional space, the introduction of reference points was deemed unnecessary as the plane crowding distance adequately addressed two-objective optimization problems, while reference points would have increased computational costs. Nevertheless, further improvement is possible. For instance, with a high number of objectives or divisions per objective, the number of generated reference points becomes substantial, leading to increased computational demands. Further research is warranted to explore methods for improving reference point generation and optimizing its efficiency.

Based on the experimental results for static functions, MF-DMOLSO demonstrated superiority over classical MOPSO, NSGA-II, and NSGA-III in both two-objective and three-objective problems, effectively balancing the approximation and distribution indices of search. Experiments on two dynamic multi-objective optimization functions showed that MF-DMOLSO could dynamically adjust under different environments, increasing initial population diversity to enhance algorithm tracking and exhibiting excellent adaptability and convergence speed. In the context of 6R robot trajectory planning, MF-DMOLSO successfully optimized polynomial trajectory parameters, resulting in reduced running time and maximum acceleration, thereby enhancing the robot's operational smoothness and efficiency. This optimization is significant in reducing mechanical jitter and improving production efficiency and safety. In this application, MF-DMOLSO outperformed MOPSO in terms of speed and robustness, highlighting the algorithm's optimization efficiency and reliability.

	\section{Conclusion}
This study introduces the Dynamic Multi-Objective Lion Swarm Optimization with Multi-strategy Fusion (MF-DMOLSO) algorithm, addressing the limitations of LSO in dynamic multi-objective optimization, such as susceptibility to local optima and slow convergence.  MF-DMOLSO incorporates Tent chaotic mapping to enhance initial population distribution, facilitating comprehensive search space coverage and accelerating the discovery of optimal frontiers. Improvements to the position update formula for lion cubs were made, following a strategy of broad search followed by fine-tuning near the PFs, thus enhancing flexibility and search capability.

The algorithm introduces self-adaptive cold and hot start strategies to recognize environmental changes and integrate historical population data, effectively tackling dynamic issues. To maintain population diversity and prevent local optima, the concept of crowding degree was introduced, varying by dimension: planar distance in two-objective space and linkage to predefined reference points in multi-objective space. These reference points guide population evolution and ensure an even distribution in the objective space.  Additionally, the Levy flight mutation mechanism was incorporated to enhance the ability to escape local optima. By integrating crowding degree sorting and Levy mutation into position updates and global optimal solution selection, unexplored PF regions received greater optimization focus.  MF-DMOLSO significantly improved early convergence speed and late-stage global search capability, achieving high optimization success rates and uniform frontier solution distribution.

The algorithm was tested on eight benchmark functions, demonstrating superior convergence properties across performance indicators compared to MOPSO, NSGA-II, NSGA-III, MOLSO, RMOLSO, and MC-DCMOEA.  Furthermore, MF-DMOLSO was applied to the joint trajectory planning of a six-axis robot using a 3-5-3 polynomial trajectory equation. This application reduced running time to 8.3s and maximum acceleration to 54°/s², avoiding motor overload, reducing mechanical dither and wear, and ensuring smoother and safer operation. Experimental and simulation results confirmed MF-DMOLSO's efficacy and feasibility in multi-objective decision-making.

	\backmatter
	\bmhead{CRediT authorship contribution statement}
Bao Liu: Writing – original draft, Methodology, Formal analysis. Tianbao Liu: Writing – review \& editing, Formal analysis, Conceptualization. Zhongshuo Hu: Visualization, Validation. Fei Ye: Investigation, Project administration. Lei Gao: Supervision, Investigation, Validation.
\bmhead{Data availability }

Data will be made available on request.

\bmhead{Acknowledgments}

This paper is supported by the Shandong Provincial Natural Science Foundation (Grant No. ZR2021MF105).

\section*{Declaration}
\bmhead{Competing Interest}
The authors declare that they have no known competing financial interests or personal relationships that could have appeared to influence the work reported in this paper.

	\begin{appendices}
\section{Other experimental results}
The following is the supplementary material related to this article.
\renewcommand{\thetable}{A.\arabic{table}}
\setcounter{table}{0}
\begin{table}[htbp]
	\caption{Performance comparison between MF-DMOLSO and other algorithms for 2-objective problems when dimension D=10. Boldfaced values represent the best performance achieved.}
	\begin{tabular}{ccccc}
		\hline
		\textbf{Function(D=10)} & \textbf{Performance index} &\textbf{MOPSO} & \textbf{NSGA-II} & \textbf{MF-DMOLSO}\\
		\hline
		ZDT1 & GD & 9.8657E-4 & 7.8496E-4 & \textbf{5.0882E-4}\\
		& $\Delta$ & 8.2258E-3 & \textbf{0.0051} & 0.0069\\
		& ER & 0\% & 0\% & \textbf{0\%}\\
		ZDT2 & GD & 9.6966E-4 & 6.9258E-4 & \textbf{5.0959E-4}\\
		& $\Delta$ & 0.0083 & 0.0067 & \textbf{0.0059}\\
		& ER & 0\% & 0\% & \textbf{0\%}\\
		ZDT3 & GD & 0.0093 & 7.2838E-4 & \textbf{5.5155E-4}\\
		& $\Delta$ & 0.0097 & \textbf{0.0056} & 0.0061\\
		& ER & 51.4231\%& 1.6558\% & \textbf{0\%}\\
		\hline
	\end{tabular}
\end{table}

\begin{table}[htbp]
	\caption{Performance comparison between MF-DMOLSO and other algorithms for 2-objective problems when dimension D=30. Boldfaced values represent the best performance achieved.}
	\begin{tabular}{ccccc}
		\hline
		\textbf{Function(D=30)} & \textbf{Performance index} &\textbf{MOPSO} & \textbf{NSGA-II} & \textbf{MF-DMOLSO}\\
		\hline
		ZDT1 & GD & 0.0011 & 0.0086 & \textbf{4.7608E-4}\\
		& $\Delta$ & 0.0080 & 0.0067 & \textbf{0.0066}\\
		& ER & 23.2635\% & 100\% & \textbf{0\%}\\
		ZDT2 & GD & 0.0013 & 0.0074 & \textbf{4.8562E-4}\\
		& $\Delta$ & 0.0107 & 0.0059 & \textbf{0.0056}\\
		& ER & 60.3960\% & 100\% & \textbf{0\%}\\
		ZDT3 & GD & 0.0112 & 0.0103 & \textbf{5.4627E-4}\\
		& $\Delta$ & 0.0104 & \textbf{0.0060} & 0.0066\\
		& ER & 99.0099\% & 100\% & \textbf{5.6655\%}\\
		\hline
	\end{tabular}
\end{table}

\begin{table}[htbp]
	\caption{Optimization results of MF-DMOLSO in G2 problem}
	\begin{tabular}{cccccc}
		\hline
		\textbf{Iterations} & \textbf{Index} & \textbf{T=0.3} & \textbf{T=0.5} & \textbf{T=0.7}& \textbf{T=0.9}\\
		\hline
		2500 & GD & 1.2643E-3 & 8.8677E-4 & 6.8778E-4 & 7.2550E-4\\
		& $\Delta$ & 0.0309 & 0.0058 & 0.0058 & 0.0059\\
		& ER & 51\% & 22\% & 18\% & 18\%\\
		5000 & GD & 8.2004E-4 & 6.1045E-4 & 5.8498E-4 & 5.9059E-4\\
		& $\Delta$ & 0.0931 & 0.0431 & 0.0431 & 0.0431\\
		& ER & 46\% & 41\% & 40\% & 43\%\\
		10000 & GD & 5.7535E-5 & 5.3398E-5 & 5.6098E-5 & 5.7959E-5\\
		& $\Delta$ & 0.0061 & 0.0052 & 0.0057 & 0.0059\\
		& ER & 3\% & 2\% & 2\% & 2\%\\
		15000 & GD & 2.3535E-5 & 1.6936E-5 & 8.9986E-6 & 5.2331E-5\\
		& $\Delta$ & 0.0053 & 0.0051 & 0.0054 & 0.0055\\
		& ER & 0\% & 0\% & 0\% & 0\%\\
		20000 & GD & 1.0894E-5 & 1.0208E-5 & 9.6889E-6 & 1.7555E-5\\
		& $\Delta$ & 0.0054 & 0.0053 & 0.0052 & 0.0052\\
		& ER & 0\% & 0\% & 0\% & 0\%\\
		\hline
	\end{tabular}
\end{table}

\clearpage
\section{Robot modeling and trajectory information}
The following is the robot modeling and trajectory information related to this article.  
\renewcommand{\thetable}{B.\arabic{table}}
\setcounter{table}{0}
\begin{table}[htbp]
	\caption{Cartesian space pose of the target point}
	\begin{tabular}{ccc}
		\hline
		\textbf{Target point} & \textbf{Spatial position ($\mathbf{x},\mathbf{y},\mathbf{z}$)} & \textbf{Space gesture($\mathbf{R},\mathbf{P},\mathbf{V}$)}\\
		\hline
		$P_1$ & (740,487,185)& (-30,7,155)\\
		$P_2$ & (1184,366,342)& (7,27,177)\\
		$P_3$ & (1235,88,340)& (-71,23,-172)\\
		$P_4$ & (960,-244,122)& (-46,3,-173)\\
		\hline
	\end{tabular}
\end{table}

\begin{table}[htbp]
	\caption{STEP SR-1400 D-H parameters}
	\begin{tabular}{ccccc}
		\hline
		\textbf{Joint i} & \textbf{$\mathbf{a_{i-1}}$/(mm)} & \textbf{$\mathbf{\alpha_{i-1}}$/(\si{\degree})} & \textbf{$ \mathbf{d_i}$/(mm)} & \textbf{$\mathbf{\theta_i}$/(\si{\degree})}\\
		\hline
		1 &180&-90&415&$\mathbf{\theta_1}$\\
		2 &590&0&0&$\mathbf{\theta_2}$(-90)\\
		3 &115&-90&0&$\mathbf{\theta_3}$\\
		4 &0&90&625&$\mathbf{\theta_4}$\\
		5 &0&-90&0&$\mathbf{\theta_5}$\\
		6 &0&0&98&$\mathbf{\theta_6}$\\
		\hline
	\end{tabular}
\end{table}

$i$ is the number of the joint axis, $\alpha_{i-1}$ and $a_{i-1}$\ indicate the torsion angle and length of the connecting rod firmly connected with the $i^{th}$ axis, $d_i$\ is the deviation distance of the connecting rod, and $\theta_i$ is the range of the joint rotation angle. The inverse kinematics were obtained according to the D-H parameter table, and the joint angle values of the robot axes 1-6 at the four target points $P_1\sim P_4$\ are displayed in Table B.3. 

\begin{table}[htbp]
	\caption{Joint angle values of target points}
	\begin{tabular}{ccccccc}
		\hline
		&\textbf{Joint 1}(\si{\degree}) & \textbf{Joint 2}(\si{\degree})&\textbf{Joint 3}(\si{\degree})&\textbf{Joint 4}(\si{\degree})&\textbf{Joint 5}(\si{\degree})&\textbf{Joint 6}(\si{\degree})\\
		\hline
		$P_1$ &36&-45&27&-45&36&45\\
		$P_2$ &18&-36&-12&-25&25&45\\
		$P_3$ &4&-36&-12&15&25&-78\\
		$P_4$ &-15&-36&15&21&21&-81\\
		\hline
	\end{tabular}
\end{table}
\end{appendices}
	\bibliography{mynewref}

\end{document}